\documentclass[conference]{IEEEtran}
\usepackage{amsmath}
\usepackage{cite}
\usepackage{multicol}
\usepackage{multirow}
\usepackage{color}
\usepackage{algpseudocode}  
\usepackage{booktabs}       
\usepackage{bbm}
\usepackage{amsfonts}       
\usepackage{nicefrac}  
\usepackage{amsthm}
\usepackage{microtype}      
\usepackage{xcolor}         
\usepackage{amsmath}
\usepackage{mathtools}
\usepackage{booktabs}

\usepackage{geometry}
\usepackage{caption}
\usepackage[ruled,vlined]{algorithm2e}
\usepackage{graphicx}
\usepackage{array}
\usepackage{amsmath, bm} 
\usepackage{varwidth}

\newtheorem{theorem}{Theorem}

\newtheorem{lemma}[theorem]{Lemma}

\graphicspath{ {./} }
\usepackage{caption}
\usepackage{subcaption}
\usepackage{graphicx}
\graphicspath{ {./} }

\def\BibTeX{{\rm B\kern-.05em{\sc i\kern-.025em b}\kern-.08em
    T\kern-.1667em\lower.7ex\hbox{E}\kern-.125emX}}

\begin{document}

\title{Optimal Resource Allocation for U-Shaped Parallel Split Learning}

\author{
  \IEEEauthorblockN{Song Lyu\IEEEauthorrefmark{1}, Zheng Lin\IEEEauthorrefmark{1}, Guanqiao Qu\IEEEauthorrefmark{1}, Xianhao Chen\IEEEauthorrefmark{1}, Xiaoxia Huang\IEEEauthorrefmark{2}, and Pan Li\IEEEauthorrefmark{3}}
  \IEEEauthorblockA{\IEEEauthorrefmark{1}The Department of Electrical and Electronic Engineering, The University of Hong Kong, Pok Fu Lam, Hong Kong, China}
\IEEEauthorblockA{\IEEEauthorrefmark{2}School of Electronics and Communication Engineering, Sun Yat-sen University, Shenzhen 510275, China}
\IEEEauthorblockA{\IEEEauthorrefmark{3}The Department of Electrical, Computer, and System Engineering, Case Western Reserve University, \\Cleveland, OH 44106 USA.}
    
}


\maketitle
\begin{abstract}
Split learning (SL) has emerged as a promising approach for model training without revealing the raw data samples from the data owners. However, traditional SL inevitably leaks label privacy as the tail model (with the last layers) should be placed on the server. To overcome this limitation, one promising solution is to utilize U-shaped architecture to leave both early layers and last layers on the user side. In this paper, we develop a novel parallel U-shaped split learning and devise the optimal resource optimization scheme to improve the performance of edge networks. In the proposed framework, multiple users communicate with an edge server for SL. We analyze the end-to-end delay of each client during the training process and design an efficient resource allocation algorithm, called LSCRA, which finds the optimal computing resource allocation and split layers. Our experimental results show the effectiveness of LSCRA and that U-shaped parallel split learning can achieve a similar performance with other SL baselines while preserving label privacy.
\end{abstract}

 \begin{IEEEkeywords}
U-shaped network, split learning, label privacy, resource allocation, 5G/6G edge networks.
 \end{IEEEkeywords}

\section{Introduction}
Traditional centralized learning incurs excessive bandwidth consumption and communication latency while violating data privacy. To address this issue, edge learning, which trains models at the network edge, has emerged as a promising paradigm in 5G and beyond\cite{hou2022edge,peng2023energy,hou2021machine,peng2021leopard}. In this respect, federated edge learning (FEEL) \cite{konevcny2016federated},\cite{chen2022federated} has been shown as an effective approach that enables end devices to train models on their own devices and then aggregate the models at an edge server, thereby eliminating the need for access the raw data.

However, FEEL faces significant challenges due to the extensive client-side computing workload. For massive resource-constrained IoT devices, the limited computing power may hinder their ability to perform model training and upload large models\cite{imteaj2021survey}. To address these challenges, split learning (SL) \cite{vepakomma2018split} has emerged as an effective technique. SL splits the model into two parts: the front sub-model (head model) trained by a client and the remaining sub-model (tail model) trained by a server\cite{gupta2018distributed}. As a result, SL significantly relieves clients' computing burden by allowing a server to take over the major workload while remaining raw data on the client side~\cite{lin2023split,lin2023efficient,lin2023pushing}.

There exist several popular SL approaches. Vanilla SL has limited scalability due to its sequential training manner~\cite{vepakomma2018split}, where the model training can be shifted to the next client only when the previous client completes training. To parallize SL, parallel split learning (PSL) \cite{jeon2020privacy} enables parallel processing across the server and multiple connected clients. Furthermore, split federated learning (SFL) \cite{thapa2022splitfed} integrates federated learning (FL) into SL to allow parallel training. U-shaped split federated learning (U-SFL) \cite{yin2023predictive} combines the U-shaped architecture with the SFL framework to eliminate label sharing. Compared to PSL, the major change in SFL lies in the averaging of the client-side sub-model after its backpropagation process, following the spirit of FL, yet incurring additional communication overhead due to model exchange. The comparison of these approaches is summarized in Table I.

By preserving users' raw data, SL is often considered in privacy-sensitive applications~\cite{yang2022robust}. Nevertheless, despite the preservation of input data, the sharing of label can be a serious privacy concerns in SL, as clients have to provide the corresponding labels to help the server to calculate the loss. In some applications, label privacy is an important concern, particularly in healthcare, finance, and other sensitive domains. For example, the input data can be users' bio information/activities, and the label is the disease or health status of this user. In this case, the label is also highly sensitive and should not be shared with the server. 

To address the label privacy issue, U-shaped configurations has been proposed for SL to eliminate the need for label exchange~\cite{vepakomma2018split}. In the U-shaped SL architecture, the entire DNN is divided into three submodels: the head, body, and tail models. The head and tail models are obtained on the client side, while the body model is trained on the edge server side. This architecture effectively resolves the label privacy concern, as the output layer and the labels are retained on the client side. Although U-shaped SL has been studied under various contexts, such as medical applications~\cite{yang2022robust}, to our best knowledge, very few efforts have been made to integrate U-shaped SL into the mobile edge.

In this paper, we investigate \underline{U}-Shaped \underline{P}arallel \underline{S}plit \underline{L}earning (U-PSL) under the mobile edge computing framework. This framework parallelizes the vanilla U-shaped SL by enabling multiple clients to train with a server simultaneously. Furthermore, we develop the joint model split and resource allocation problem tailored for U-shaped SL, called LSCRA. By formulating the per-round training latency, we obtain the optimal server computing resource allocation and layer splitting strategy to address the communication and computing challenges associated with U-shaped networks, resulting in a significant reduction in training latency. Through experiments, it is found that U-PSL achieves effective label privacy protection while achieving similar or even slightly shorter latency compared to other SL benchmarks, making it a promising solution for SL in privacy-sensitive and resource-constrained wireless networks.

Our contributions are summarised as follows: 
\begin{itemize}
\item We propose U-PSL, an advanced privacy-enhancing training framework, which eliminates the need for raw data sharing and label sharing in SL.

\item We design an optimal joint computing resource allocation and layer splitting scheme to minimize per-round latency. 

\item We conduct simulations to demonstrate the effectiveness of the U-PSL framework. Our simulations show the effectiveness of the resource allocation scheme, revealing that the framework achieves test accuracy comparable to other SL approaches while preserving label privacy.
\end{itemize}

\begin{table}[t]
	\caption \centering{The Comparison of FL, SL, SFL, PSL, U-SFL, and U-PSL Frameworks}
 \centering
	\renewcommand\arraystretch{0.8}{
	\setlength{\tabcolsep}{1.3mm}{
    
	\begin{tabular}{ccccccc}
    \toprule[1.5pt] 
    Learning framework & FL & SL & SFL & PSL& U-SFL & U-PSL\\
    \midrule[1pt]
    Computation offloading & No & Yes & Yes & Yes & Yes & Yes\\
	\midrule[0.5pt]
    Parallel computing & Yes & No & Yes & Yes & Yes & Yes\\
	\midrule[0.5pt]
    Access to raw data & No & No & No & No & No & No\\
	\midrule[0.5pt]
	Model exchange & Yes & No & Yes & No & Yes & No\\
	\midrule[0.5pt]
	Label sharing & No & Yes & Yes & Yes & No & No\\
    \bottomrule[1.5pt]
    \end{tabular}}}
\end{table}

\section{System Model and U-PSL Framework}
\vspace{-0.1cm}
This section presents the U-PSL framework, which is illustrated in Figure \ref{system model}. We begin by describing a scenario of the U-PSL framework in wireless networks. Subsequently, we provide a detailed explanation of the five main steps in the U-PSL workflow. Through this section, we aim to provide an overview of the U-PSL framework and its step-by-step training procedure. Furthermore, since a shorter training time not only enables timely model usage but also reduces bandwidth and computing resource occupation, we will analyze and optimize the end-to-end latency. For the convenience of readers, we summarize the important notations in Table II.

\textbf{Architecture:} U-PSL comprises an edge server and multiple clients. On the client side, we assume that each client has an end device with computing capabilities, enabling it to execute forward propagation (FP) and backpropagation (BP) for the client-side models. Let $\mathcal U=\{1,2,...,N\}$ denote the set of clients, where $N$ is the number of participating clients. The local dataset $D_n$ owned by client $n\in \mathcal U$ is represented as $D_n= \{X_{n},Y_{n} \}$, where $X_{n}$ denotes the $n$-th client's training dataset and $Y_{n}$ is the set of the corresponding labels. $\rho_j$ and $\omega_j$ denote the computing workload of FP and BP for the first $j$ layers, respectively, $\psi_j$ represents the activation size at cut layer $j$ in the model, and $L$ denotes the total number of model layers.

%

\textbf{U-PSL Workflow:} Figure \ref{U-PSL framework} illustrates the main workflow of U-PSL, which consists of five training steps:

\subsubsection{Head model FP \& activations transmission} At the beginning of model training, the server initializes the global model and partitions it into three submodels $W_{head}$, $W_{body}$, and $W_{tail}$. $\mu^j_1=\{0,1\}$ and $\mu^j_2=\{0,1\}$ indicates the two split layers between head models and body models, and body models and tail models. 
$\mu^j_1=1$ indicates that layer $j$ is the first cut layer, and $\mu^j_2=1$ indicates that layer $j$ is the second cut layer. At the beginning of each round, each client randomly draws a mini-batch $\beta_n$ (generally, the size can be proportional to the size of $X_n$) to perform the FP process of the head model in parallel.
For simplicity, we focus on client $n$ to illustrate the operations on the client side. Let $\varphi^F_1(\overline{\mu^j_1})$ denote the computing workload of the head model’s FP process for one data sample, which is given by: 
\begin{equation}
\varphi^F_1(\overline{\mu^j_1})=\sum^L_{j=1} \mu^j_1\rho_j.
\end{equation}
\begin{figure}[t]
\centering
         \includegraphics[width=0.33\textwidth]{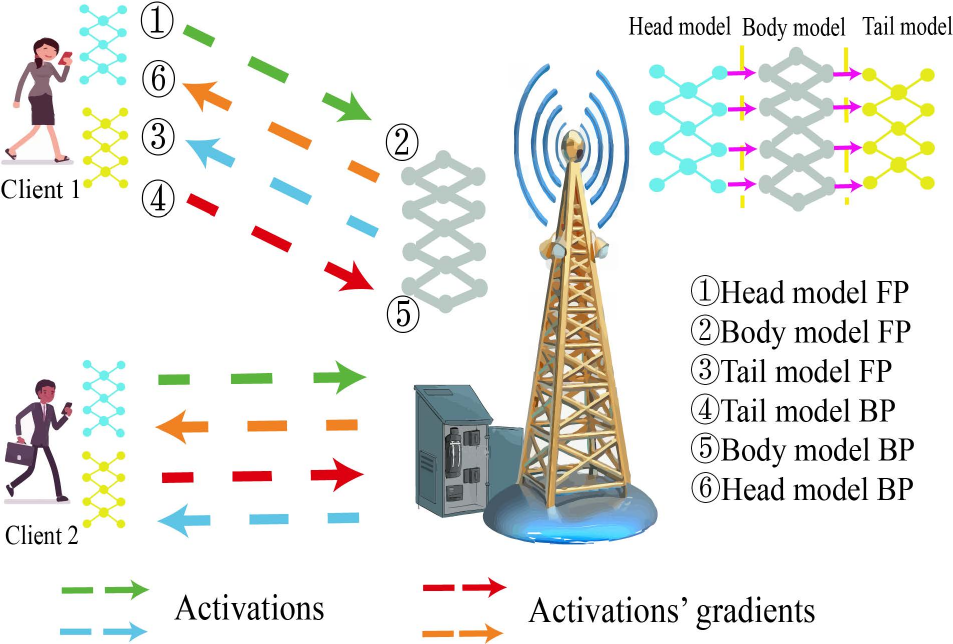}
     \caption{The illustration of U-PSL over wireless networks.}
	\label{system model}
\end{figure}

After completing the head model FP process, the first cut layer generates activations that will be taken as the input of the body model on the server. Then, the client transmits the activations to the server over a wireless channel. The data size of the activations $\Gamma(\overline{\mu^j_1})$ can be expressed as: 
\begin{equation}
\Gamma(\overline{\mu^j_1})=\sum^L_{j=1} \mu^j_1\psi_j.\label{activationsize1}
\end{equation}
\vspace{-0.1cm}

Therefore, the latency of step 1 for client $n$ can be denoted as: 
\begin{equation}
t^c_{1,n}=\frac{\beta_n\varphi^F_1(\overline{\mu^j_1})K_c}{f_n}+\frac{\beta_n\Gamma(\overline{\mu^j_1})}{R^{\uparrow}_n},
\end{equation}
where $f_n$ is the computing capability of client $n$, $K_c$ is the computing intensity of client, and $R^{\uparrow}_n$ is the upload data rate. We consider a static network where the average data rate does not change, and therefore $R^{\uparrow}_n$ is a constant value. The mobility scenarios can be left for the future research~\cite{chen2023vehicle,ding2021context,lin2022channel,lin2022tracking}.

\begin{table}[t]
	\caption{Frequently Used Notations}
	\renewcommand\arraystretch{1.1}{
	\setlength{\tabcolsep}{1mm}{
    \centering
	\begin{tabular}{ll}
    \toprule[1.5pt]
    Notation & Interpretation\\
    \midrule[1pt]
	$\mathcal U$ & The set of clients\\
    $D_n$ & Local dataset of client $n$\\
    $\beta_n$ & mini-batch size draw from client $n$'s local dataset\\
    $R^{\uparrow}_n /R^{\downarrow}_n$ & Upload/download data rate of client $n$\\
	$f_n$ & The computing capability of client $n$\\
	$f_{s,n}$ & The server-side computing resource allocated for client $n$\\
	$K_s/K_c$ & The computing intensity of server/client $n$\\
	$L$ & The total number of model layers in CNN\\
	$\rho_j$ & The computation workload (in CPU cycles) of FP for \\
 &the first $j$ layers\\
	$\omega_j$ & The computation workload (in CPU cycles) of BP for \\
 &the first $j$ layers\\
	$\psi_j$ & The size of activations (or activations' gradients) \\
 &of the cut layer $j$\\
	$F_s$ & Maximum computing capability of the server\\
	\toprule[1.5pt]
    \end{tabular}}}
\end{table}

\subsubsection{Body model FP \& activations transmission} When the server receives the activations from clients, the body model starts its FP process. $\varphi^F_s(\overline{\mu^j_1},\overline{\mu^j_2})$ denotes the computation workload of the body model’s FP process for one data sample, which can be described as:
\begin{equation}
\varphi^F_s(\overline{\mu^j_1},\overline{\mu^j_2})=\sum^L_{j=1} \mu^j_2\rho_j-\sum^L_{j=1} \mu^j_1\rho_j.
\end{equation}

After the completion of each mini-batch, the second cut layer generates activations, the size of which can be expressed as: 
\begin{equation}
\Gamma(\overline{\mu^j_2})=\sum^L_{j=1} \mu^j_2\psi_j. \label{activationsize2}
\end{equation}

Therefore, the latency of step 2 for client $n$ can be denoted as: 
\begin{equation}
t^s_{2,n}=\frac{\beta_n\varphi^F_s(\overline{\mu^j_1},\overline{\mu^j_2})K_s}{f_{s,n}}+\frac{\beta_n\Gamma(\overline{\mu^j_2})}{R^{\downarrow}_n}.
\end{equation}
where $f_{s,n}$ is the server computing resource allocation of client $n$, $K_s$ is the computing intensity of server, and $R^{\downarrow}_n$ is the download data rate.

\subsubsection{Tail model FP and BP \& activations' gradients transmission} At this stage, the client performs the rest FP process of the tail model to calculate the loss and then conducts the BP process. Let $\varphi^F_2(\overline{\mu^j_2})$  and $\varphi^B_2(\overline{\mu^j_2})$ represent
the computation workload of the tail model’s FP and BP process, respectively, which can be described as: 
\begin{gather}
\varphi^F_2(\overline{\mu^j_2})=\rho_L-\sum^L_{j=1}\mu^j_2\rho_j,\\
\varphi^B_2(\overline{\mu^j_2})=\omega_L-\sum^L_{j=1}\mu^j_2\omega_j.
\end{gather}

After finishing the BP process, each client sends activations' gradients back to the server. Given the size of the activations' gradients at the second layer $\Gamma(\overline{\mu^j_2})$ in (\ref{activationsize2}), the latency of step 3 for client $n$ can be denoted as: 
\begin{equation}
t^c_{3,n}=\frac{\beta_nK_c(\varphi^F_2(\overline{\mu^j_2})+\varphi^B_2(\overline{\mu^j_2}))}{f_n}+\frac{\beta_n\Gamma(\overline{\mu^j_2})}{R^{\uparrow}_n}.
\end{equation}
\begin{figure}[t]
\centering
         \includegraphics[width=0.42\textwidth]{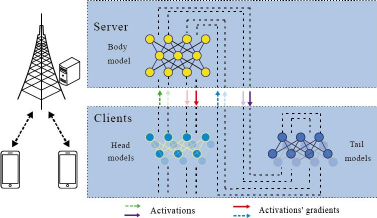}
     \caption{U-PSL Framework.}
	\label{U-PSL framework}
\end{figure}

\subsubsection{Body model BP \& activations' gradients transmission} After receiving activations’ gradients, the server performs its BP process. Let $\varphi^B_s(\overline{\mu^j_1},\overline{\mu^j_2})$ denotes the computation workload of the body model’s BP process, which can be described as: 
\begin{equation}
\varphi^B_s(\overline{\mu^j_1},\overline{\mu^j_2})=\sum^L_{j=1}\mu^j_2\omega_j-\sum^L_{j=1}\mu^j_1\omega_j.
\end{equation}
\vspace{-0.1cm}

When the body model’s BP process is completed, activations' gradients at the first cut layer will be transmitted to the corresponding clients. The size of activations' gradients is $\Gamma(\overline{\mu^j_1})$ in (\ref{activationsize1}), and therefore the latency of step 4 for client $n$ can be denoted as: 
\begin{equation}
t^s_{4,n}=\frac{\beta_n\varphi^B_s(\overline{\mu^j_1},\overline{\mu^j_2})K_s}{f_{s,n}}+\frac{\beta_n\Gamma(\overline{\mu^j_1})}{R^{\downarrow}_n}.
\end{equation}

\subsubsection{Head model BP} In this stage, the client only needs to complete the rest BP process of the head model. $\varphi^B_1(\overline{\mu^j_1})$ denotes the computation workload of the head model’s BP process, which can be described as: 
\begin{equation}
\varphi^B_1(\overline{\mu^j_1})=\sum^L_{j=1}\mu^j_1\omega_j.
\end{equation}

Therefore, the latency of step 5 for client $n$ can be denoted as: 
\begin{equation}
t^c_{5,n}=\frac{\beta_n\varphi^B_1(\overline{\mu^j_1})K_c}{f_{s,n}}.
\end{equation}

After the aforementioned steps, each sub-model updates the model parameters according to the gradients. Note that, for the body model, the server can make updates based on the averaged gradients across the clients. The per-round training latency corresponding to client $n$  can be denoted as: 
\begin{equation}
T_n=t^c_{1,n}+t^s_{2,n}+t^c_{3,n}+t^s_{4,n}+t^c_{5,n}.
\end{equation}

Let $T(f, \mu_1, \mu_2)$ denote the per-round training time. Since the aforementioned training is executed in parallel, $T(f, \mu_1, \mu_2)$ is equal to the maximum $T_n$, i.e., 
\begin{equation}
T(f, \mu_1, \mu_2)=\max_{n\in\mathcal U}  T_n.
\end{equation}

\section{Problem Formulation and Solution Approach}
As mentioned earlier, the total latency of one training round for a client is formulated. Apparently, inappropriate server computing resource allocation can lead to significant increases in training time. Additionally, the selection of cut layers also affects the overall training and communication latency. Considering these factors, we formulate the following optimization problem to minimize the per-round latency: 
\begin{align}
\mathcal{P}1:&\mathop {{\rm{min}}}\limits_{{\bm{f}},{\bm{\mu_1}},{\bm{\mu_2}}} T({\bf{f}},{\bm{\mu_1}},{\bm{\mu_2}})   \\
&\mathrm{s.t.} ~\mathrm{C1:}~\sum\limits_{j' = 1}^j {\mu_2^{j'}}\leq \sum\limits_{j' = 1}^j {\mu_1^{j'}}, \forall j \in \{1, ..., L\} \nonumber, \\
&~\mathrm{C2:}~{\mu_2^j}\in\{0,1\}, {\mu_1^j}\in\{0,1\}, \forall j \in \{1, ..., L\} \nonumber, \\
&~\mathrm{C3:}~\sum\limits_{j = 1}^L {\mu_1^j}=1, \sum\limits_{j = 1}^L {\mu_2^j}=1          \nonumber, \\
&~\mathrm{C4:}~f_{s,n} \geq 0, \forall n \in {\mathcal U} \nonumber, \\
&~\mathrm{C5:}~\sum\limits_{n = 1}^N {f_{s,n}}\leq {F_s} \nonumber.
\end{align}
where $\mathrm{C1}$ ensures that the index of the second split layer is greater than the index of the first split layer. To solve $\mathcal{P}1$, we first consider the subproblem involving computing resource allocation: 
\begin{align}
\mathcal{P}2:&\mathop {{\rm{min}}}\limits_{{\bm{f}}} T({\bf{f}})   \\
&\mathrm{s.t.} ~\mathrm{C4:}~f_{s,n} \geq 0, \forall n \in {\mathcal U} \nonumber, \\
&~\mathrm{C5:}~\sum\limits_{n = 1}^N {f_{s,n}}\leq {F_s} \nonumber.
\end{align}

We have the following lemmas for $\mathcal{P}2$.

\begin{lemma}
The optimal $\bf{f}$ for $\mathcal{P}2$ is obtained when $T_1=\cdots=T_n$.\label{lemma1}
\end{lemma}
\begin{proof}
Let $f_{s,n}=f^*_{s,n}$ be the solution that minimizes the objective while satisfying $T_1=\cdots=T_n$. Assume $T_1=\cdots=T_n=\overline T$ and therefore $T({\bf{f}})=\underset{n} {\max}T_n=\overline T$ in this case. It can be shown that $\sum\limits_{n = 1}^N {f^*_{s,n}} = {F_s}$. Otherwise, if $\sum\limits_{n = 1}^N {f^*_{s,n}} < {F_s}$, the remaining resources can be evenly allocated to every $f^*_{s,n}$, thereby reducing the objective. Supposing that there is $T_m > \overline T$, we have $T({\bf{f}}) \geq T_m > \overline T$. On the other hand, if there is $T_m < \overline T$, we have $f_{s,m}>f^*_{s,m}$. Thus, there must be $f_{s,n}<f^*_{s,n}$ since $\sum\limits_{n = 1}^N {f^*_{s,n}} = {F_s}$. Hence, we have $T_n > \overline T$, leading to $T({\bf{f}})\geq T_n > \overline T$. Therefore, only when $T_1=\cdots=T_n=\overline T$, the optimal resource allocation can be obtained. The proof is completed.
\end{proof}

\begin{lemma}
The $k$-th client with the maximum allocated computing resource $f_{s,k}$ should satisfy the equation:
\begin{equation}
\label{sum fsk expression}
\sum^N_{n=1}\frac{\varepsilon_n f_{s,k}}{\varepsilon_k+f_{s,k}(T^{local}_k-T^{local}_n)}=F_s.
\end{equation}
\end{lemma}
\begin{proof}
When $\mu^j_1$ and $\mu^j_2$ are fixed, each training epoch latency can be described as $T_n = T^{local}_n + \frac{\varepsilon_n}{f_{s,n}}$, where $\varepsilon_n = \beta_nK_s (\sum^L_{j=1}\mu^j_2\rho_j+\sum^L_{j=1}\mu^j_2\omega_j-\sum^L_{j=1}\mu^j_1\rho_j-\sum^L_{j=1}\mu^j_1\omega_j)$ denotes the server-side computing workload, and $T^{local}_n$ is a constant representing client's local computing and communication latency. 

For client set $\mathcal U$, by enforcing $T_1=\cdots=T_k=T_n=\overline T, k\in \mathcal U$, the equation can be given as
\begin{equation}
T^{local}_n+\frac{\varepsilon_n}{f_{s,n}}=T^{local}_k+\frac{\varepsilon_k}{f_{s,k}}, \forall {k, n}\in {\mathcal U}.
\end{equation}
Therefore, to achieve equal per-round training time, we have
\begin{equation}
f_{s,n}=\frac{\varepsilon_n f_{s,k}}{\varepsilon_k+f_{s,k}\left(T^{local}_k-T^{local}_n\right)}.\label{fsn expression}
\end{equation}

To satisfy C4 in $\mathcal{P}2$, the selected $k$-th client should be the one with the maximum $T^{local}_n$ to ensure $f_{s,n}$ is nonnegative. Besides, as discussed in Lemma \ref{lemma1}, $\sum^N_{n=1}f_{s,n}=F_s$ holds for the optimal solution. By considering (\ref{fsn expression}), we have
\begin{equation}
\sum^N_{n=1}\frac{\varepsilon_n f_{s,k}}{\varepsilon_k+f_{s,k}(T^{local}_k-T^{local}_n)}=F_s.
\end{equation}
\end{proof}
\vspace{-0.4cm}

We observe that Eq.  (\ref{sum fsk expression}) exhibits a monotonically increasing behavior with respect to $f_{s,k}$. Taking advantage of this property, we can employ a bisection procedure to efficiently find $f_{s,k}$ from (\ref{sum fsk expression}). Then, the optimal $f_{s,n}$ for other clients can be directly obtained from (\ref{fsn expression}).

After obtaining the optimal server computing resource allocation scheme, the remaining task involves making split-layer decisions. This subproblem can be formulated as: 
\begin{align}
\mathcal{P}3:&\mathop {{\rm{min}}}\limits_{{\bm{\mu_1}},{\bm{\mu_2}}} T({\bm{\mu_1}},{\bm{\mu_2}})   \\
&\mathrm{s.t.} ~\mathrm{C1:}~\sum\limits_{j' = 1}^j {\mu_2^{j'}}\leq \sum\limits_{j' = 1}^j {\mu_1^{j'}}, \forall j \in \{1, ..., L\} \nonumber, \\
&~\mathrm{C2:}~{\mu_2^j}\in\{0,1\}, {\mu_1^j}\in\{0,1\}, \forall j \in \{1, ..., L\} \nonumber, \\
&~\mathrm{C3:}~\sum\limits_{j = 1}^L {\mu_1^j}=1, \sum\limits_{j = 1}^L {\mu_2^j}=1          \nonumber. 
\end{align}
\vspace{-0.07cm}
~~$\mathcal{P}3$ is a standard mixed integer linear programming (MILP) problem. Since the number of CNN model layers is typically not very large, we can directly use an exhaustive search algorithm to calculate the minimum $T_n$ and obtain $\mu^j_1$ and $\mu^j_2$. 

Finally, our proposed scheme, termed Layer Splitting and Computing Resource Allocation (LSCRA), conducts exhaustive search to ensure that all possible pairs of split layers are explored. Then, with each pair, we solve the optimal resource allocation based on bisection procedure from Eq.  (\ref{sum fsk expression}) to find the minimum delay attained. It is easy to see that LSCRA can obtain the optimal solution to $\mathcal{P}1$, and the computational complexity is $O(L^2 {log}{F_s})$.


\begin{figure}[t]
\centering
\begin{minipage}{0.21\textwidth}
\centering
\includegraphics[width=\textwidth]{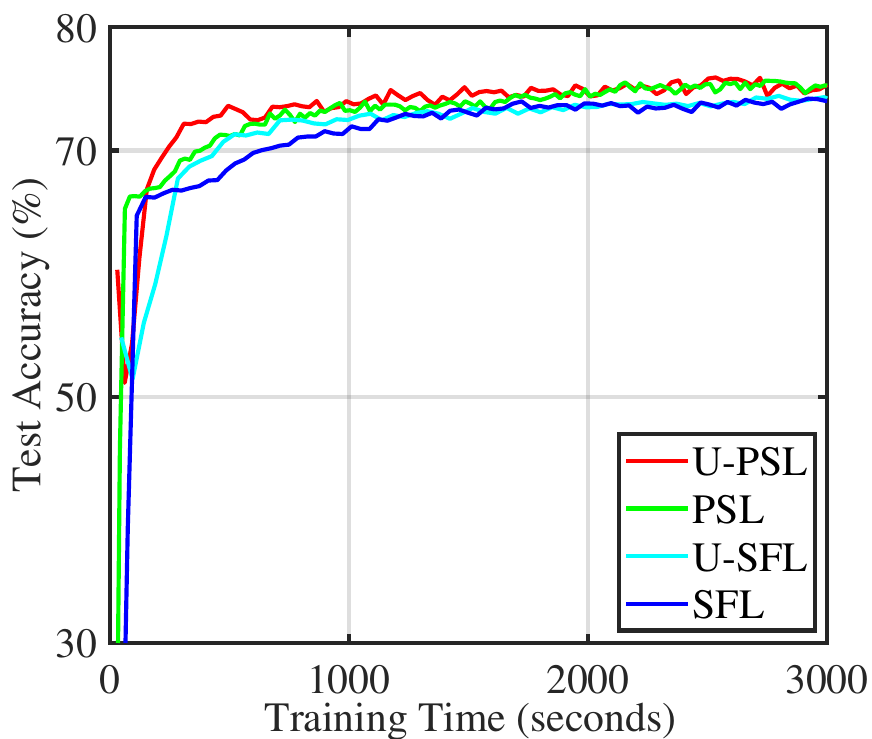}
\caption*{(a) HAM10000 under IID setting}
\end{minipage}\hfill
\begin{minipage}{0.21\textwidth}
\centering
\includegraphics[width=\textwidth]{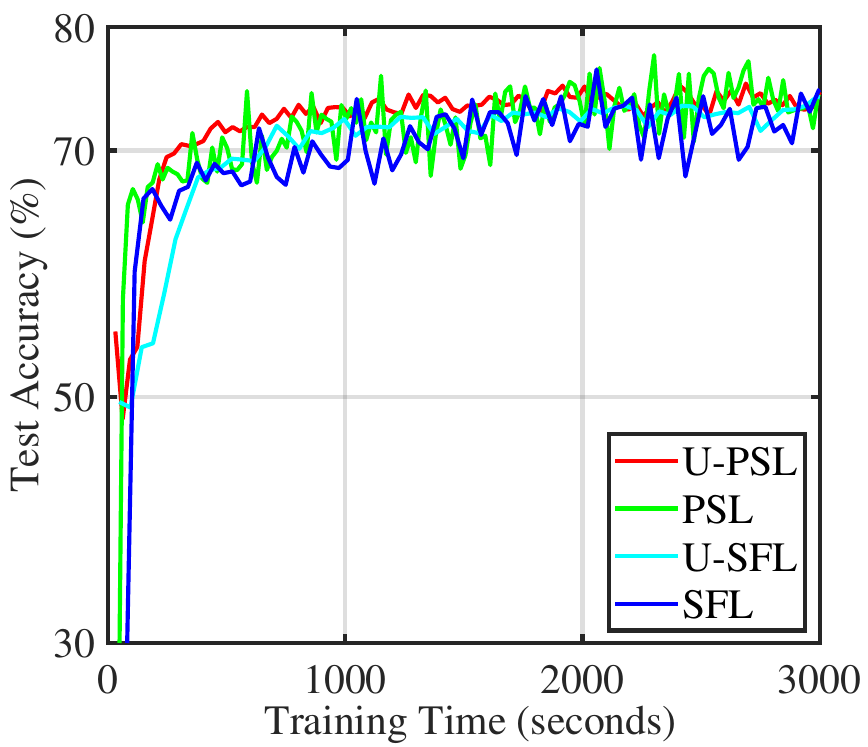}
\caption*{(b) HAM10000 under non-IID setting}
\end{minipage}\hfill
\begin{minipage}{0.21\textwidth}
\centering
\includegraphics[width=\textwidth]{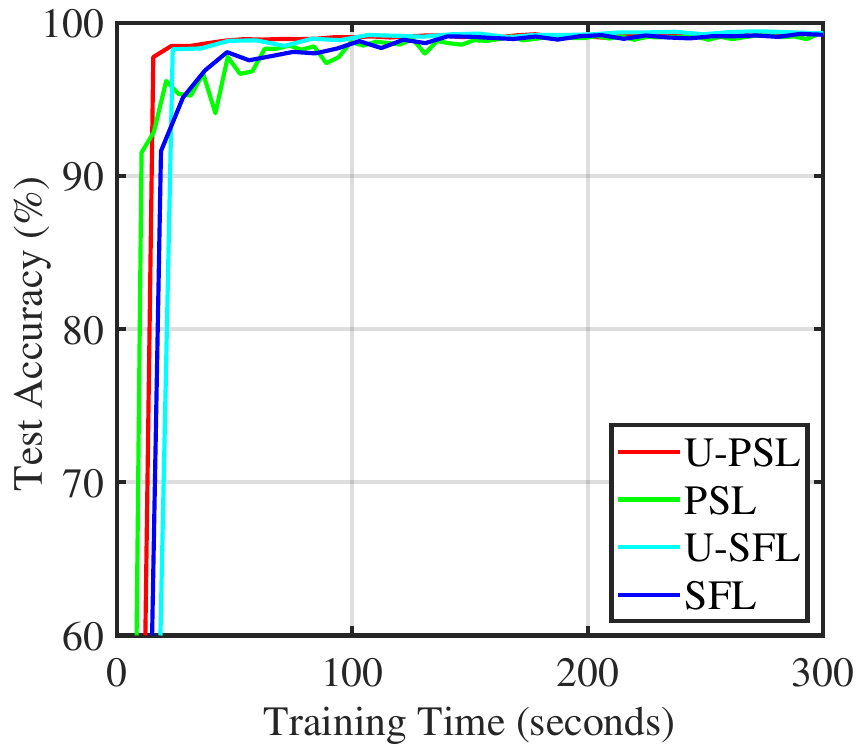}
\caption*{(c) MNIST under IID setting
\newline}
\end{minipage}\hfill
\begin{minipage}{0.21\textwidth}
\centering
\includegraphics[width=0.99\textwidth]{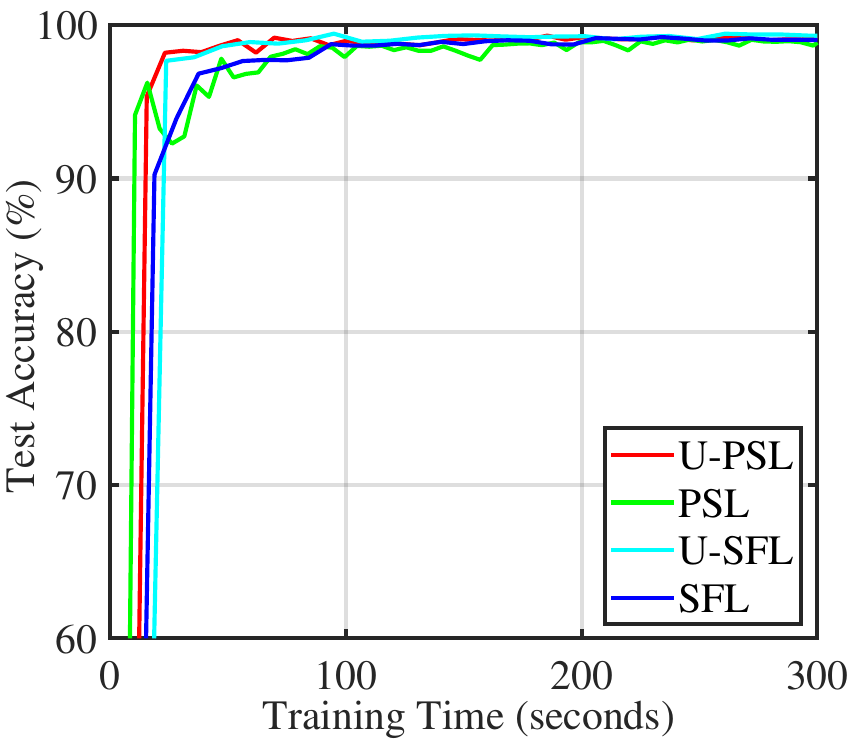}
\caption*{(d) MNIST under non-IID setting}
\end{minipage}
\vspace{-0.3cm}
\caption{Test accuracy of U-PSL, PSL, U-SFL, SFL on HAM10000 \& MNIST dataset under IID/non-IID setting with $N=5$, $F_s=50$GHz.}
\label{test accuracy}
\end{figure}

\section{Simulation Results}
This section provides the numerical results to evaluate the learning performance of the proposed U-PSL framework and the effectiveness of the LSCRA algorithm and split layers strategy. 

\subsection{Experiments Settings}
In the simulations, we consider $N$ clients randomly distributed around a wireless edge server. The computing capability of each client is uniformly distributed within $[0.5, 1.5]$ GHz, and the computing capability of the server is set to $[10, 50]$ GHz. The uplink data rate of each client is uniformly distributed within $[5, 30]$ Mbps, and the downlink data rate is set to $[2,10] \times R^\uparrow_n$ Mbps. Other parameters can be found in Table III.

We evaluate the learning performance of the proposed U-PSL framework by deploying the ResNet-18 network on two image classification datasets, HAM10000\cite{tschandl2018ham10000} and MNIST\cite{lecun1998gradient}. Furthermore, we conduct experiments under IID (independent and identically distributed) and non-IID data settings. 

\begin{table}[h]
  \centering
  \caption{Parameter Settings}
  \renewcommand{\arraystretch}{1.25}{
  \setlength{\tabcolsep}{0.5mm}{
\begin{tabular}{|c|c|c|c|}
\hline
\textbf{Parameter}          & \textbf{value} & \textbf{Parameter} & \textbf{value}  \\ \hline
$F_s$             & $[10, 50]$GHz              & $f_n$                 & $[0.5, 1.5]$GHz \\ \hline
$N$             & [5, 100]              & $\beta_n$           & 64                    \\ \hline
$K_s$               & $\frac{1}{32}$cycles/FLOPs              & $K_c$                  & $\frac{1}{16}$cycles/FLOPs                        \\ \hline
$R^\uparrow_n$        & $[5, 30]$Mbps            & $R^\downarrow_n$             & $[2, 10]\times{R^\uparrow_n}$Mbps                       \\ \hline
\end{tabular}}}
\end{table}
\vspace{-0.2cm}

\subsection{Performance Evaluation of the Proposed U-PSL Framework}
In this subsection, we assess the performance of the proposed U-PSL framework in terms of test accuracy, convergence speed, training latency, and privacy preservation. We compare U-PSL with other distributed learning baselines, including PSL, SFL, and U-SFL, to examine the effectiveness of U-PSL. For fair comparison, the benchmark schemes also adopt optimal split layers and server computing resource allocation.

Figure \ref{test accuracy} demonstrates the test accuracy of these frameworks on the HAM10000 and MNIST datasets. It can be observed that U-PSL achieves a similar test accuracy compared to SFL, U-SFL and PSL as the models converge. Moreover, in some situations (e.g., Figure \ref{test accuracy}(a)), U-PSL requires the lowest time budget to reach a target accuracy. There are two reasons for this: One is that the client-side submodels in U-PSL are trained by user-specific data. Therefore, the client-side submodel may adapt better to user data in the early stages and perform better in terms of accuracy. The other is that U-PSL eliminates the need for model exchange between the clients and the server, reducing communication overhead and resulting in faster convergence compared to U-SFL and SFL.

\begin{figure}[t]
\centering
         \includegraphics[width=0.47\textwidth]{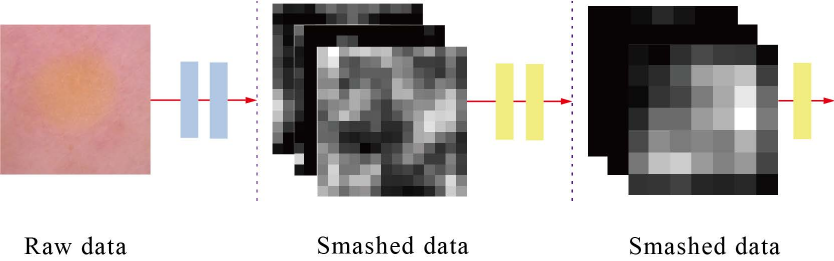}
     \caption{Smashed data visualization.}
	\label{smashed data}
\end{figure}

Figure \ref{smashed data} illustrates the use of a raw image from HAM10000 to generate smashed data at the first and second cut layers, which are located after the skip connection of the third and fourth residual blocks in Resnet-18, respectively. From the visualization, the outputs significantly differ from the raw data. Also, it is hard to identify the label. In summary, \textit{the U-PSL framework achieves both data and label privacy protection while achieving similar or even better performance compared to other benchmarks. }

\subsection{Performance Evaluation of the Proposed LSCRA algorithm}
In this subsection, we evaluate the performance of the proposed LSCRA scheme with respect to the server computing capacity and the number of service clients. We compare the proposed method with two benchmarks: 
\begin{itemize}
\item \textbf{Benchmark a)}: \emph{Optimal split layers \& evenly allocated}, where the server and clients have the same cut layers as the proposed scheme, and the server computing resource is evenly allocated. 
\item\textbf{Benchmark b)}: \emph{Suboptimal split layers \& evenly allocated}, where the cut layers are set to the second performing case, and the server computing resource is evenly allocated.
\end{itemize}

Figure \ref{server computing capacity} illustrates the performance of the per-round training latency with respect to the server computing capacity. It can be observed that when the server's computing capacity is limited, the proposed scheme significantly reduces the training latency for each round. This is achieved by allocating more server computing resources to devices with weaker computing power and communication conditions. 

Furthermore, when the server's computing capacity ranges from $10$ GHz to $50$ GHz, the proposed scheme ensures that the training time for each round does not decrease significantly. This is because, in scenarios where the server's computing capacity is sufficiently powerful, the communication time and the local training time of clients become the dominant factors. However, even in such cases, our method outperforms benchmark b), by finding the optimal split layers. This phenomenon demonstrates the importance of carefully selecting splitting layers and allocating computing resources. In a nutshell, our method reduces the training latency with varied computing capabilities, particularly in scenarios where the resources on the server are limited. 

\begin{figure}[t]
\centering
         \includegraphics[width=0.21\textwidth]{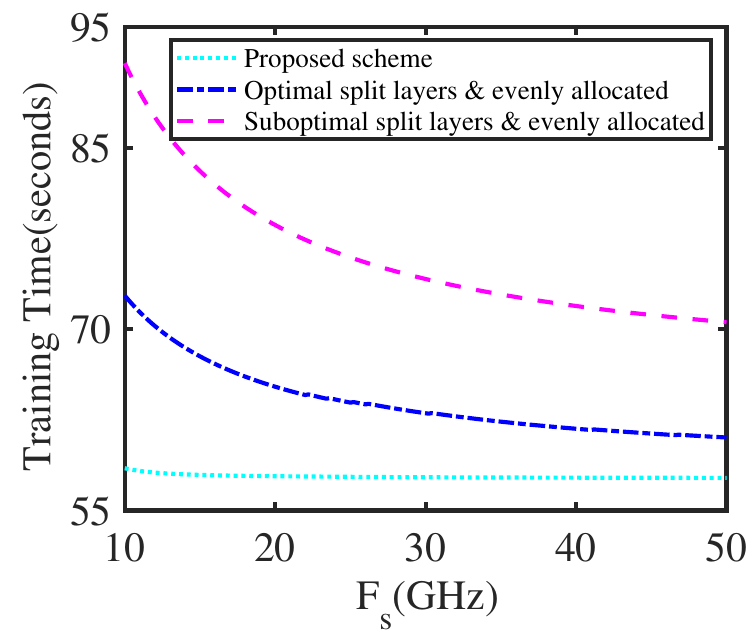}
     \caption{The performance for per-round training latency versus the server computing capacity with $[10, 50]$ GHz, $N=100$.}
\label{server computing capacity}
\end{figure}

\begin{figure}[t]
\centering
         \includegraphics[width=0.21\textwidth]{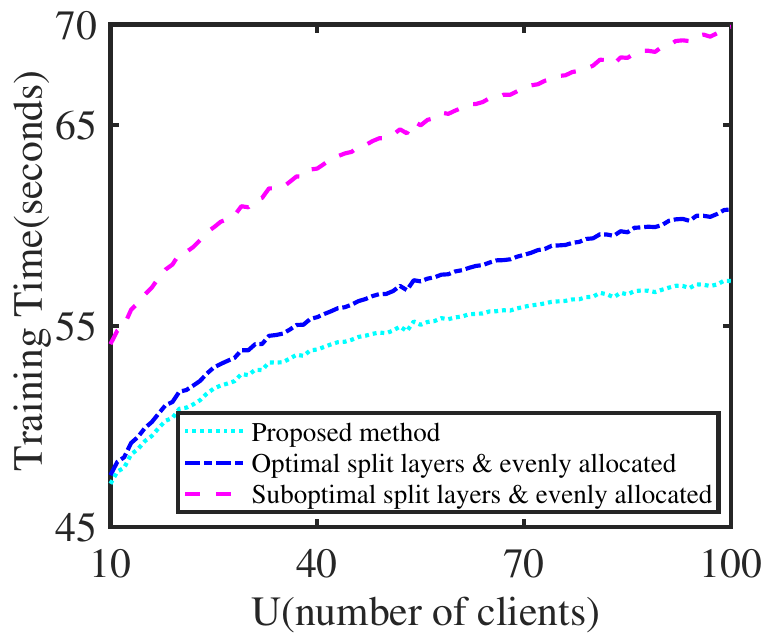}
     \caption{The performance for per-round training latency versus the number of clients from 10 to 100, $F_s=50$ GHz}
\label{number of clients}
\end{figure}

Figure \ref{number of clients} illustrates the performance of the per-round training latency with respect to the number of clients. As the number of clients increases, the time cost for each round associated with the two benchmarks shows a greater increase compared to our proposed scheme. This scenario aligns with real-world communication scenarios where a single server serves a large number of users.
\vspace{-0.1cm}
\section{Conclusions}
\vspace{-0.1cm}
In this paper, we proposed a novel split learning framework called U-Shaped Parallel Split Learning (U-PSL) to address model and label privacy preservation. By taking into account the additional communication overhead introduced by the U-shaped neural network, we have designed an effective resource allocation and layer splitting strategy to minimize the latency of U-PSL over wireless edge networks. Simulation results demonstrate that our proposed U-PSL framework retains a similar accuracy compared to existing SL benchmarks while preserving label privacy. Our results show the effectiveness and efficiency of adopting U-shaped SL at wireless edge networks. For the future work, we plan to derive the convergence results for our scheme and consider the joint optimization of computing resources and channel allocation for U-shaped PSL.

\vspace{-0.15cm}
\section{Acknowledgment}
\vspace{-0.1cm}
The work of X. Chen was supported in part by HKU IDS Research Seed Fund under grant IDS-RSF2023-0012. The work of X. Huang was supported by Joint Funds of NSFC under grant U22A2003.
\bibliographystyle{IEEEtran}
\bibliography{mybib}

\end{document}